\documentclass{article}
\usepackage{spconf,amsmath,graphicx,cite}
\usepackage{amsmath}
\usepackage{amssymb}
\usepackage{amsthm}
\usepackage{algorithm}
\usepackage{algorithmic}
\usepackage{tikz}
\usepackage{tabu}
\usepackage{graphics}
\usepackage{subfig}
\usepackage{epstopdf}
\usepackage{epsfig}
\usepackage{caption}
\usepackage{float}
\newtheorem{theorem}{Theorem}

\newtheorem{definition}[theorem]{Definition}

\newcommand{\new}[1]{{#1}} 

\def\R{{\mathbb{R}}}

\usepackage{enumitem}

\usepackage{todonotes}

\DeclareMathOperator*{\argmax}{arg\,max}

\title{Demixing Structured Superposition Signals \\ from Periodic and Aperiodic Nonlinear Observations}

\name{Mohammadreza Soltani and Chinmay Hegde\thanks{
The first part of this manuscript appeared in the (non-archival) workshop paper \cite{NLDemix_NIPSworkshop}. This work is supported in part by NSF grants CCF-1566281 and IIP-1632116 and an NVIDIA GPU grant.
}
}
\address{ECpE Department, Iowa State University, Ames, IA, 50010}
\begin{document}
\ninept
\maketitle
\begin{abstract}

We consider the demixing problem of two (or more) structured high-dimensional vectors
from a limited number of nonlinear observations where this nonlinearity is due to either a periodic or an aperiodic function. We study certain families of structured superposition models, and propose a method which provably recovers the components given (nearly) $m = \mathcal{O}(s)$ samples where $s$ denotes the sparsity level of the underlying components. This strictly improves upon previous nonlinear demixing techniques and asymptotically matches the best possible sample complexity. We also provide a range of simulations to illustrate the performance of the proposed algorithms.
\end{abstract}
%
%
\section{Introduction}
\label{Intro}
\subsection{Motivation}
The \emph{demixing} problem involves disentangling two (or more) high-dimensional vectors from their linear superposition, and has several applications in signal and image processing, statistics, and data analysis~\cite{mccoyTropp2014,mccoy2014convexity,soltani2016fastTSP17,SoltaniHegde_Asilomar,SoltaniHegde_Globalsip}. In applications involving signal recovery, such superpositions can be used to model situations when there is some ambiguity in the components (e.g., the true components can be treated as ``ground truth'' + ``outliers'') or when there is some existing prior knowledge that the true underlying vector is a superposition of two components. Mathematically, suppose that the underlying signal is given by $\beta =\Phi\theta_1 + \Psi \theta_2$ where $\beta,\theta_1, \theta_2\in\mathbb{R}^n$ and $\Phi, \Psi$ are orthonormal bases. If a linear observation model is assumed, then given measurements $y\in\mathbb{R}^m$ and a design matrix $X \in \mathbb{R}^{m \times n}$, the goal is to recover the signal $\beta$ that minimizes a loss function $\mathcal{L}(X,y; \beta)$. We focus on the sample-poor regime where the dimension far exceeds the number of measurements; this regime has received significant attention from the machine learning and signal processing communities in recent years~\cite{negahban2009unified,CandesCS}.  

However, fitting the observations according to a linear model can be sometimes restrictive depending on the application. One way to ease this restriction is to assume a \emph{nonlinear} observation model:
\begin{align}\label{nonlindex}
 y = g(X\beta) + e =g(X(\Phi\theta_1 + \Psi \theta_2) )+ e,
 \end{align}
where $g$ denotes a nonlinear \textit{link} function and $e$ denotes observation noise. This is akin to the \emph{Generalized Linear Model} (GLM) as well as the \emph{Single Index Model} (SIM) from statistical learning~\cite{kakade2011}. Here, the problem is to estimate $\theta_1$ and $\theta_2$ from the observations $y$ with as few measurements as possible. The estimation problem~\eqref{nonlindex} is challenging in several different aspects: \\
(i) there is a basic identifiability of issue of obtaining $\theta_1$ and $\theta_2$ even with perfect knowledge of $\beta$; \\
(ii) there is a second identifiability issue arising from the nontrivial null-space of the design matrix (since $m \ll n$); and \\
(iii) the nonlinear nature of $g$, as well as the presence of noise $e$ can further confound recovery.

Standard techniques to overcome each of these challenges are well-known. By and large, these techniques all make some type of \emph{sparseness} assumption on the components $\theta_1$ and $\theta_2$ \cite{CandesCS}; some type of \emph{incoherence} assumption on the bases $\Phi$ and $\Psi$ \cite{elad2005simultaneous,donoho2006stable}; and some regularity condition on $g$ (which we elaborate later). However, to our knowledge, the confluence of the three above challenges have not been simultaneously addressed in the literature.

\subsection{Summary of contributions}
In this paper, we focus on the case where the components $\theta_1,\theta_2$ obey certain \emph{structured sparsity} assumptions. Structured sparsity models are useful in applications where the support patterns (i.e., the coordinates of the nonzero entries) belong to model-specific restricted families (for example, the support is assumed to be \emph{group-sparse} \cite{huang2010benefit}). It is known that such assumptions can significantly reduce the required number of samples for estimating the underlying signal, compared to generic sparsity assumptions~\cite{modelcs,SPINIT,approxIT}. We consider two classes of link functions:\emph{aperiodic} and \emph{periodic} functions, and accordingly, two different demixing approaches. Our approach in this paper builds upon and extends our recent previous work on nonlinear demixing~\cite{soltani2016fastTSP17,SoltaniHegde_Asilomar,SoltaniHegde_Globalsip}.

 In the aperiodic case, we follow the setup of~\cite{soltani2016fastTSP17} where $g$ is assumed to be monotonic; satisfies some type of \emph{restricted strong convexity} (RSC)~\cite{negahban2009unified}; and some type of \emph{restricted strong smoothness} (RSS) assumptions~\cite{yang2015sparse}. For this case, we develop a non-convex iterative algorithm that stably estimates the components $\theta_1$ and $\theta_2$.
 
 In the periodic case, we use the approach of~\cite{SoltaniHegde_ICASSP16} both for designing the matrix $X$ and the link function $g$. Specifically, we let $X$ be \emph{factorized} as $X = D B$, where $D\in\mathbb{R}^{m\times q}$, and $B\in\mathbb{R}^{q\times n}$ have some specific structures; please see Section \ref{Perm} for details. Again, for this case, we demonstrate a novel two-stage algorithm that stably estimate the components $\theta_1$ and $\theta_2$. 
 
For both cases considered above, we show that under certain sufficiency conditions, the performance of our methods strictly improves upon previous nonlinear demixing techniques, and asymptotically matches (close to) the best possible sample-complexity. 

\subsection{Prior work}
\label{sec:prior}
The demixing problem has been a recent focus in several fields including signal and image processing, machine learning, and computational physics~\cite{mccoyTropp2014}. The majority of the literature on the demixing problem studies the case of linear superposition of two or more components where these components can be modeled in various ways such as sparse vectors~\cite{mccoy2014convexity}, low-rank and sparse matrices~\cite{candes2011rpca}, and manifold models~\cite{spinisit,SPINIT}. Recently, a few papers have addressed the nonlinear setting where the observations are index-wise nonlinear functions of the superposition of the components~\cite{soltani2016fastTSP17,SoltaniHegde_Asilomar,SoltaniHegde_Globalsip}. This nonlinear demixing framework can also be considered as a special instance of nonlinear signal recovery which has recently received broad attention~\cite{boufounos20081,planRomanLin,davenport20141,ganti2015matrix}. For instance, \cite{soltani2016fastTSP17} considers the nonlinear demixing of a pair of sparse vectors with arbitrary supports where the nonlinearity is a monotonic function, obtains a sufficient condition on the number of samples for achieving a desired estimation accuracy. On the other hand,~\cite{SoltaniHegde_ICASSP16} studies the problem of nonlinear signal recovery and demixing, where the nonlinearity is a periodic function. 

We note that demixing approaches in high dimensions with structured sparsity assumptions have appeared before in the literature~\cite{mccoyTropp2014,mccoy2014convexity,rao2014forward}. However, our method differs from these earlier works in a few different aspects. The majority of these methods involve solving a convex relaxation problem; in contrast, our algorithm is manifestly \emph{non-convex}. Despite this feature, for certain types of structured superposition models, our method provably recovers the components given (nearly) $m = \mathcal{O}(s)$ samples. Moreover, these earlier methods have not explicitly addressed the nonlinear observation model (with the exception of \cite{plan2014high}). In this paper, we leverage the structured sparsity assumptions to our advantage, and show that this type of structured sparsity priors significantly decreases the sample complexity (both for periodic and aperiodic nonlinearities) for estimating the signal components.

\section{Preliminaries}
\label{Perm}
Let $\|.\|_q$ denote the $\ell_q$-norm of a vector. Denote the spectral norm of the matrix $X$ as $\|X\|$. Denote the true parameter vector, $\theta =  [\theta_1;\theta_2 ]\in\mathbb{R}^{2n}$ as the vector obtaining by stacking the true and unknown coefficient vectors,  $\theta_1, \theta_2$. For simplicity of exposition, in this paper we suppose that the components $\theta_1$ and $\theta_2$ exhibit \emph{block} sparsity with sparsity $s$ and block size $b$~\cite{modelcs}. (Analogous approaches apply for other structured sparsity models.)

The problem~\eqref{nonlindex} is inherently ill-posed. To resolve this issue, we need to assume that the coefficient vectors $\theta_1, \theta_2$ are somehow distinguishable from each other. This is characterized by a notion of incoherence of the components $\theta_1, \theta_2$~\cite{SoltaniHegde_Globalsip}.

\begin{definition}\label{incoherence}
The bases $\Phi$ and $\Psi$ are called $\varepsilon$-incoherent if 
$
\varepsilon = \sup_{\substack{\|u\|_0\leq s,\ \|v\|_0\leq s  \\ \|u\|_2 = 1,\ \|v\|_2 = 1}}|\langle{\Phi u, \Psi v}\rangle|.
$
\end{definition}
For the analysis of aperiodic link functions, we need the following standard definition~\cite{negahban2009unified}:
\begin{definition}\label{rssrsc}
A function $f : \mathbb{R}^{2n} \rightarrow \mathbb{R}$ satisfies {Structured} \textit{Restricted Strong Convexity/Smoothness {(SRSC/SRSS)} }if:
\begin{align*}
m_{4s}\leq\|\nabla^2_{\xi} f(t)\|\leq M_{4s},\  \ t\in\mathbb{R}^{2n},
\end{align*}
where $\xi = \textrm{supp}(t_1)\cup \textrm{supp}(t_2)$, for all $t_i\in\mathbb{R}^{2n}$ such that \new{$t_i$ belongs to $(2s,b)$ block-sparse vectors} 
for $i=1,2$, and $m_{4s}$ and $M_{4s}$ are (respectively) the \new{SRSC and SRSS} constants. Also, $\nabla^2_{\xi} f(t)$ denotes a $4s\times 4s$ sub-matrix of the Hessian matrix $\nabla^2 f(t)$ comprised of rows/columns indexed by $\xi \subset [2n]$.
\end{definition}
Furthermore, for aperiodic functions, we assume that the derivative of the link function is strictly bounded either within a positive interval, or within a negative interval. 
In addition, let $\beta_j$ denotes the $j^{\mathrm{th}}$ entry of the signal $\beta\in\mathbb{R}^n$.  Also, for $j \in \{1,2,\ldots,q\}$, $\beta(j:q:(k-1)q+j)\in\mathbb{R}^k$ denotes the sub-vector of $\beta$, starting at index $j + qr$, where $r = 0, 1, \ldots, k-1$. Finally, $Y((j:q:(k-1)q,l)$ represents the sub-vector made by picking the $l^{\mathrm{th}}$ column of any matrix $Y$ and choosing the entries of this column as stated. 

For the analysis of periodic link functions, by following the approach of~\cite{SoltaniHegde_ICASSP16}, we let the design matrix $X$ be \emph{factorized} as $X = D B$, where $D\in\mathbb{R}^{m\times q}$, and $B\in\mathbb{R}^{q\times n}$. We assume that $m$ is a multiple of $q$, and that $D$ is a concatenation of $k$ diagonal matrices of $q \times q$ such that the diagonal entries in the blocks of $D$ are i.i.d.\ random variables distributed uniformly within an interval $[-T,T]$ for some $T>0$. The choice of $B$ is flexible and can be chosen such that it supports stable demixing. In particular, as~\cite{soltani2016fastTSP17} has shown, $B$ can be any random matrix with independent subgaussian rows. Overall, our low-dimensional observation model can be written as:
\begin{align}\label{StruDMF}
y = g(DB\beta) + e = g(DB(\Phi\theta_1 + \Psi \theta_2)) + e,
\end{align}
where $g$ is either sinusoidal function, or any periodic function such that in each period, it behaves monotonically. Furthermore, $D = [D_1,\ldots,D_k]^T$ comprises $k$ diagonal matrices $D_i$'s, and $e\in\mathbb{R}^m$ denotes additive noise such that $e\sim\mathcal{N}(0,\sigma^2I)$. The goal is to stably recover $\theta_1, \theta_2$ from the embedding $y$. The diagonal structure of the matrix $D$ reduces the the final recovery of underlying components to first obtaining a good enough estimation of $B\beta$, and then using a linear demixing approach from a (possibly noisy) estimate of $B\beta$ will lead to the estimation of $\theta_1, \theta_2$.

\section{Algorithms and analysis}
In this section, we describe our algorithm and theoretical result for both aperiodic and periodic link functions.

\subsection{Aperiodic link functions}
To solve the demixing problem in~\eqref{nonlindex}, we consider the minimization of a special loss function $F(t)$, following~\cite{SoltaniHegde_Globalsip}:
\begin{equation} \label{optprob}
\begin{aligned} \underset{t \in \mathbb{R}^{2n}}{\text{min}}
F(t) &= \frac{1}{m}\sum_{i=1}^m \Theta(x_i^T\Gamma t) - y_i x_i^T\Gamma t \ \ \ \
 \text{s.\ t.} \  \  t\in\mathcal{D},
\end{aligned}
\end{equation}
where $\Theta'(x) = g(x)$, $\Gamma = [\Phi \  \Psi]$, $x_i$ is the $i^{\textrm{th}}$ row of the design matrix $X$ and $\mathcal{D}$ denotes the set of length-$2p$ vectors formed by stacking a pair of $(s,b)$ block-sparse vectors. The objective function in~\eqref{optprob} is motivated by the single index model in statistics; for details, see~\cite{SoltaniHegde_Globalsip}. To approximately solve~\eqref{optprob}, we propose an algorithm which we call it \emph{Structured Demixing with Hard Thresholding} (STRUCT-DHT). The pseudocode of this algorithm is given in Algorithm~\ref{algHTM}.
\begin{algorithm}[t]
\caption{Structured Demixing with Hard Thresholding (STRUCT-DHT)
\label{algHTM}
}
\begin{algorithmic}
\STATE \textbf{Inputs:} Bases $\Phi$ and $\Psi$, design matrix $X$, link function $g$, observation $y$, sparsity $s$, block size $b$, step size $\eta'$. 
\STATE \textbf{Outputs:} Estimates  $\widehat{\beta}=\Phi\widehat{\theta_1} + \Psi\widehat{\theta_2}$, $\widehat{\theta_1}$, $\widehat{\theta_2}$
\STATE\textbf{Initialization:}
\STATE$\left(\beta^0, \theta_1^0, \theta_2^0\right)\leftarrow\textsc{random initialization}$
\STATE$k \leftarrow 0$
\WHILE{$k\leq N$}
\STATE $t^k \leftarrow [ \theta_1^k ; \theta_2^k ]$\quad\quad\{Forming constituent vector\}
\STATE $t_1^k\leftarrow\frac{1}{m}\Phi^TX^T(g(X\beta^k) - y)$ 
\STATE$t_2^k\leftarrow\frac{1}{m}\Psi^TX^T(g(X\beta^k) - y)$
\STATE$\nabla F^k \leftarrow [ t_1^k ; t_2^k ]$
\quad\quad\{Forming gradient\}
\STATE${\tilde{t}}^k = t^k - \eta'\nabla F^k$
\quad\{Gradient update\}
\STATE$[ \theta_1^k ; \theta_2^k ]\leftarrow\mathcal{P}_{s;s;b}\left(\tilde{t}^k\right)$  
\quad\{Projection\}
\STATE$\beta^k\leftarrow\Phi \theta_1^k + \Psi \theta_2^k$\quad\{Estimating $\widehat{x}$\}
\STATE$k\leftarrow k+1$
\ENDWHILE
\STATE\textbf{Return:} $\left(\widehat{\theta_1}, \widehat{\theta_2}\right)\leftarrow \left(\theta_1^N, \theta_2^N\right)$
\end{algorithmic}
\end{algorithm}
At a high level, \textsc{STRUCT-DHT} tries to minimize loss function defined in~\eqref{optprob} (tailored to $g$) between the observed samples $y$ and the predicted responses $X\Gamma \widehat{t}$, where $\widehat{t} = [\widehat{\theta}_1; \ \widehat{\theta}_2]$ is the estimate of the parameter vector after $N$ iterations. The algorithm proceeds by iteratively updating the current estimate of $\widehat{t} \ $based on a gradient update rule followed by (myopic) \emph{hard thresholding} of the residual onto the set of $s$-sparse vectors in the span of $\Phi$ and $\Psi$. Here, we consider a version of \textsc{DHT}~\cite{SoltaniHegde_Globalsip} which is applicable for the case that coefficient vectors $\theta_1$ and $\theta_2$ have block sparsity. For this setting, we use component-wise block-hard thresholding, $\mathcal{P}_{s;s;b}$~\cite{modelcs}. Specifically, $\mathcal{P}_{s;s;b}(\tilde{t}^k)$ projects the vector $\tilde{t}^k\in\mathbb{R}^{2n}$ onto the set of concatenated $(s,b)$ block-sparse vectors by projecting the first and the second half of $\tilde{t}^k$ separately.  
Now, we provide the theorem supporting the convergence analysis and sample complexity (required number of observations for successful estimation of $\theta_1, \theta_2$) of \textsc{STRUCT-DHT}.

\begin{theorem}
\label{mainThConvergence}
Consider the observation model~\eqref{nonlindex} with all the assumption and definitions mentioned in the section~\ref{Perm}. Suppose that the corresponding objective function $F$ satisfies the Structured SRSS/SRSC properties with constants $M_{6s}$ and $m_{6s}$ 
such that $1\leq\frac{M_{6s}}{m_{6s}}\leq\frac{2}{\sqrt{3}}$ . Choose a step size parameter $\eta'$ with $\frac{0.5}{M_{6s}}<\eta^{\prime}<\frac{1.5}{m_{6s}}$. 
Then, \textsc{DHT} outputs a sequence of estimates $(\theta_1^k, \theta_1^k)$ ($t^{k+1} = [\theta_1^k; \theta_1^k]$) such that the estimation error of the underlying signal satisfies the following upper bound (in expectation) for any $k\geq 1$: 
\begin{align}
\label{eq:linconverge}
\|t^{k+1} - \theta\|_2\leq\left(2q\right)^k\|t^0-\theta\|_2 + C\tau\sqrt{\frac{s}{m}}, 
\end{align}
where $q = 2\sqrt{1+{\eta^{\prime}}^2M_{6s}^2-2\eta^{\prime} m_{6s}}$  and $C>0$ is a constant that  depends on the step size $\eta^{\prime}$ and the convergence rate $q$. Here, $\theta$ denotes the true stacked signal defined in section~\ref{Perm}.
\end{theorem}
\begin{proof}[Proof sketch]
The proof follows the technique used to prove Theorem 4.6 in~\cite{soltani2016fastTSP17}. The main steps are as follows.
Let $b'\in\mathbb{R}^{2n} =[b_1';b_2']=  t^k - \eta'\nabla F(t^k)$, $b = t^k - \eta'\nabla_J F(t^k)$ where $J :=J_k= \text{supp}(t^k)\cup \text{supp}(t^{k+1})\cup \text{supp}(\theta)$ and $b_1', \  b_2'\in\mathbb{R}^{n}$ (Here, $\theta = [\theta_1;\theta_2]$ denotes the true signal). Also define $\ t^{k+1} = \mathcal{P}_{s;s}(b') = [\mathcal{P}_s(b'_1); \mathcal{P}_s(b'_2)]$. Now, by the triangle inequality, we have:
$\|t^{k+1} - \theta\|_2\leq \|t^{k+1} - b\|_2 + \|b- \theta\|_2$.
The proof is completed by showing that $\|t^{k+1} - b\|_2\leq 2\|b - \theta\|_2$.
Finally, we use the Khintchine inequality~\cite{vershynin2010introduction} to bound the expectation of the $\ell_2$-norm of the restricted gradient function, $\nabla F(\theta)$ (evaluated at the true stacked signal $\theta$) with respect to the support set $J$).
\end{proof}

The inequality~\eqref{eq:linconverge} indicates the linear convergence behavior of our proposed algorithm. Specifically, in the noiseless scenario to achieve $\kappa$-accuracy in estimating the parameter vector $\widehat{t} = [\widehat{\theta}_1; \ \widehat{\theta}_2]$, \textsc{Struct-DHT} only requires $\log\left(\frac{1}{\kappa}\right)$ iterations.
We also have the following theorem regarding the sample complexity of Alg.\ \ref{algHTM}:
\begin{theorem}\label{samplemono}
If the rows of $X$ are independent subgaussian random vectors~\cite{vershynin2010introduction}, then the required number of samples for successful estimation of the components, $n$ is given by $\mathcal{O}\left(\frac{s}{b}\log\frac{n}{s}\right)$. Furthermore, if $b = {\Omega}\left(\log\frac{n}{s}\right)$, then the sample complexity of our proposed algorithm is given by $m = \mathcal{O}(s)$, which is asymptotically optimal.
\end{theorem}
\begin{proof}[Proof sketch]
The proof is similar to the proof of Theorem 4.8 in~\cite{soltani2016fastTSP17} where we had previously derived upper bounds on the sample complexity of demixing by proving that $F$ satisfies RSC/RSS with reasonable parameters. Here, the steps are essentially the same as in~\cite{soltani2016fastTSP17}. The proof approach uses standard concentration techniques to show that the Euclidean norm of a sparse vector with fixed support is preserved with high probability under the action of the design matrix $X$. The proof follows by taking union bound over the set of \emph{all} sparse vectors, the size of which is given by $\mathcal{O}\left((\frac{n}{s}\right)^s)$. This increases the sample complexity by a log factor over the number of ``free'' parameters. The same strategy is applicable here, except that we need to compute union bound over the set of $(s,b)$ block-sparse vectors. The size of this set is given by ${\frac{n}{b}\choose\frac{s}{b}} = \mathcal{O}\left((\frac{n}{s})^{\frac{s}{b}}\right)$ which is considerably smaller than the set of all sparse vectors. Now, if we choose $m=\mathcal{O}\left(\frac{s}{b}\log\frac{n}{s}\right)$, then the objective function in~\eqref{optprob} satisfies SRSC/SRSS condition. Finally, if $b$ scales as $b = {\Omega}\left(\log\frac{n}{s}\right)$, we obtain $m=\mathcal{O}(s)$ which is an asymptotic gain over $\mathcal{O}\left(s\log\frac{n}{s}\right)$. 
\end{proof}
The big-Oh constant hides dependencies on various parameters, including the coherence parameter $\varepsilon$, as well as the upper and the lower bounds on the derivative of the link function $g$. 

\subsection{Periodic link functions}
In this section, we focus on the periodic link functions which are either sinusoidal (complex-exponential), or any periodic function such that it is monotonic within each period. We start with the sinusoidal (complex-exponential) link function and follow the approach of~\cite{SoltaniHegde_ICASSP16}. In~\cite{SoltaniHegde_ICASSP16}, the authors proposed an algorithm called \emph{MF-Sparse} for recovering an underlying signal which is arbitrary sparse, or is the superposition of two arbitrary sparse components, but they only considered sinusoidal link functions. This algorithm has two steps: first step outputs a vector $\hat{z}$ as the estimate of $z=B\beta$ from measurement $y$  in~\eqref{StruDMF}. The idea is to use leverage the structure of the block diagonal matrix $D$ to decouple the estimation of each entry in $z$ through a tone estimation algorithm proposed in~\cite{eftekhari2013matched}. Then, $\hat{z}$ is used as the input for the second step where any sparse recovery technique can be used to estimate the underlying signal (in~\cite{SoltaniHegde_ICASSP16}, the \emph{CoSaMP} algorithm~\cite{needell2009cosamp} has been used for the second step). 

In our case, we use MF-Sparse algorithm as a core algorithm for estimating the underlying components albeit with two differences: first, we might have a preprocessing step before tone estimation depending on the periodic nonlinearity. More precisely, if we use a link function except sinusoidal, we first map the observation vector $y$ to $\tilde{y}$ through a sinusoidal function and use this new observation vector $\widetilde{y}$ as the input to the second step, tone estimation. To give a explanation why this method works, we note that in each period, the link function is assumed to be monotonic; as a result, for each entry of $\widetilde{y}$, there is one and only one entry from $y$. Thus, we can use the method of recovery under sinusoidal nonlinearity to estimate the underlying components $\widehat{\theta_1},\widehat{\theta_2}$. Second, for the third stage, we invoke STRUCT-DHT with identity link function $g(x) = x$ instead of any regular sparse recovery method. We call the resulting algorithm \emph{MF-STRUCT-DHT} and is given in Algorithm~\ref{MFSTRUCTDHT}.

\begin{algorithm}[t]
\caption{\textsc{MF-STRUCT-DHT}} 
\label{MFSTRUCTDHT}
\begin{algorithmic}
\STATE\textbf{Inputs:} $y$, $D$, $B$, $\Omega$, $s$, $b$,$\Phi$,$\Psi$,$\eta'$,$g$
\STATE\textbf{Output:} $\widehat{\theta_1},\widehat{\theta_2}$
\STATE \textbf{Stage 1: Mapping:}
\IF {$g(x)\neq\sin(x)$}
\STATE $\tilde{y} = \sin(y)$
\STATE $y\leftarrow\tilde{y}$
\ENDIF
\STATE \textbf{Stage 2: Tone estimation:}
\FOR {$l =1:q$}
\STATE $t \leftarrow D(l:q:(k-1)q+l,l)$
\STATE $u \leftarrow y(l:q:(k-1)q+l)$
\STATE $\widehat{z_l} = \argmax_{\omega\in\Omega}|\langle y,\psi_{\omega}\rangle|$
\ENDFOR
\STATE $\widehat{z} \leftarrow [\widehat{z_1},\widehat{z_2}\ldots,\widehat{z_q}]^T$
\STATE \textbf{Stage 2: Structured demixing recovery}
\STATE $g(x)\leftarrow x$
\STATE  $X\leftarrow B$
\STATE $\widehat{\theta_1},\widehat{\theta_2} \leftarrow \textsc{STRUCT-DHT}(\widehat{z},X,s,b,\Phi,\Psi,\eta',g)$
\end{algorithmic}
\end{algorithm}

%
%
By combining Theorem~\ref{samplemono} and Theorem 2.1 in~\cite{SoltaniHegde_ICASSP16}, we obtain the sample complexity of the MF-STRUCT-DHT scheme to achieve $\kappa$-accuracy.

\begin{theorem}[Sample complexity of MF-STRUCT-DHT]
Consider the measurement model in~\eqref{StruDMF} without any additive noise. Assume that the nonzero entries of block diagonal matrix $D$ are i.i.d. random variables, distributing uniformly within the interval $[-T,T]$, and the rows of $B$ are independent subgaussian random vectors (normalized by $\frac{1}{q}$). Moreover, assume $\|x\|_2 \leq R$ for some constant $R>0$. If we set $m = kq$ where $k=c_1 \log\left(\frac{Rq}{\kappa}\frac{1}{\delta}\right)$ for some $\kappa>0$, $q = \mathcal{O}\left(\frac{s}{b}\log\frac{n}{s}\right)$, $\omega = c_2 R$, and $\Omega = [-\omega,\omega]$, MF-STRUCT-DHT scheme provides an estimate $\widehat{\beta}$, such that 
$\|\beta - \widehat{\beta}\|_2 \leq \mathcal{O}( \kappa) \, ,$
with probability at least $1-\delta$. Here, $c_1, c_2$ are constants. Furthermore, if $b$ scales as $b = {\Omega}\left(\log\frac{n}{s}\right)$, then the sample complexity of MF-STRUCT-DHT scheme is given by $m = \mathcal{O}(s)$, which is asymptotically optimal.
\end{theorem}

\begin{proof}
The proof follows from a straightforward application of Theorem 2.1 in~\cite{SoltaniHegde_ICASSP16}. According to this result, one can estimate $\widehat{z}$ (the estimation of $z =B\beta$) up to $\upsilon$-accuracy if $T$ scales as $|T| = \mathcal{O}(\frac{1}{\upsilon})$. Under this choice for $T$, the required number of block diagonal matrices in $D$ to achieve $\upsilon$ accuracy for estimating $z$ is given by $k =\mathcal{O}\left(\log(\frac{\Omega}{\upsilon})\right)$ where $|\Omega| = \mathcal{O}(R)$ (see Algorithm~\ref{MFSTRUCTDHT}). Now by choosing the design matrix $B\in\R^{q\times n}$, final accuracy parameter $\kappa$ as $\upsilon =\mathcal{O}(\frac{\kappa}{\sqrt{q}})$, and choosing $q = \mathcal{O}\left(\frac{s}{b}\log\frac{n}{s}\right)$ according to Theorem~\ref{samplemono}, the result follows.
\end{proof}
Note that big-Oh constant does not depend on the bounds on the derivative of the link function since it is a identity function. In addition, if the periodic link function $g$ is set to the sinusoidal (complex-exponential), then the additive noise can be added to~\eqref{StruDMF}. In this case, the sample complexity is increased by a multiplicative factor equals to $1+\sigma^2$ where $\sigma^2$ denotes the variance of the Gaussian noise; see~\cite{SoltaniHegde_ICASSP16} for details.

\section{Numerical results}
\vskip -0.1 in
To show the efficacy of \textsc{Struct-DHT} for demixing components with structured sparsity for aperiodic link funciotns, we numerically compare \textsc{Struct-DHT} with ordinary \textsc{DHT} (which does \emph{not} leverage structured sparsity), and also with an adaptation of a convex formulation described in~\cite{yang2015sparse} that we call \emph{Demixing with Soft Thresholding} (\textsc{DST}). We first generate true components $\theta_1$ and $\theta_2$ with length $n = 2^{16}$ with nonzeros grouped in blocks with length $b = 16$ and total sparsity $s = 656$. The nonzero (active) blocks are randomly chosen from a uniform distribution over all possible blocks. 

We construct a design (observation) matrix following the construction of~\cite{krahmer2011new}. Finally, we use a (shifted) sigmoid link function given by $g(x) = \frac{1-e^{-x}}{1 + e^{-x}}$ to generate the observations $y$.  Fig~\ref{fig:ComparisonSyn} shows the the performance of the three algorithms with different number of samples averaged over $10$ Monte Carlo trials. In Fig~\ref{fig:ComparisonSyn}(a), we plot the probability of successful recovery, defined as the fraction of trials where the normalized error is less than 0.05. Fig~\ref{fig:ComparisonSyn}(b) shows the normalized estimation error for these algorithms. As we can observe, \textsc{Struct-DHT} shows much better sample complexity (the required number of samples for obtaining small relative error) as compared to \textsc{DHT} and \textsc{DST}. 

We conduct a similar experiment for two periodic link functions: sinusoidal and sawtooth (\emph{modulo}) functions with period $2\pi$ and amplitude $1$. The parameters are as before, except we set $n=2^{14}$, $s=160$, and $k=4$. We numerically compare MF-STRUCT-DHT scheme with the case where we do not consider the structured sparsity, and with a convex relaxation  formulation~\cite{yang2015sparse}. Figures~\ref{fig:ComparisonSyn}(c) and (d) show the probability of success for the sinusoidal and sawtooth cases, respectively. Again, we get the same conclusion as in the aperiodic case: our proposed algorithm achieves far improved sample complexity over previous existing methods that solely rely on sparsity assumptions.

\begin{figure}
\vskip -0.4in
\begin{center}
\begingroup
\setlength{\tabcolsep}{.1pt} 
\renewcommand{\arraystretch}{.1} 
\begin{tabular}{cc}      
\includegraphics[trim = 7mm 58mm 14mm 30mm, clip, width=0.48\linewidth]{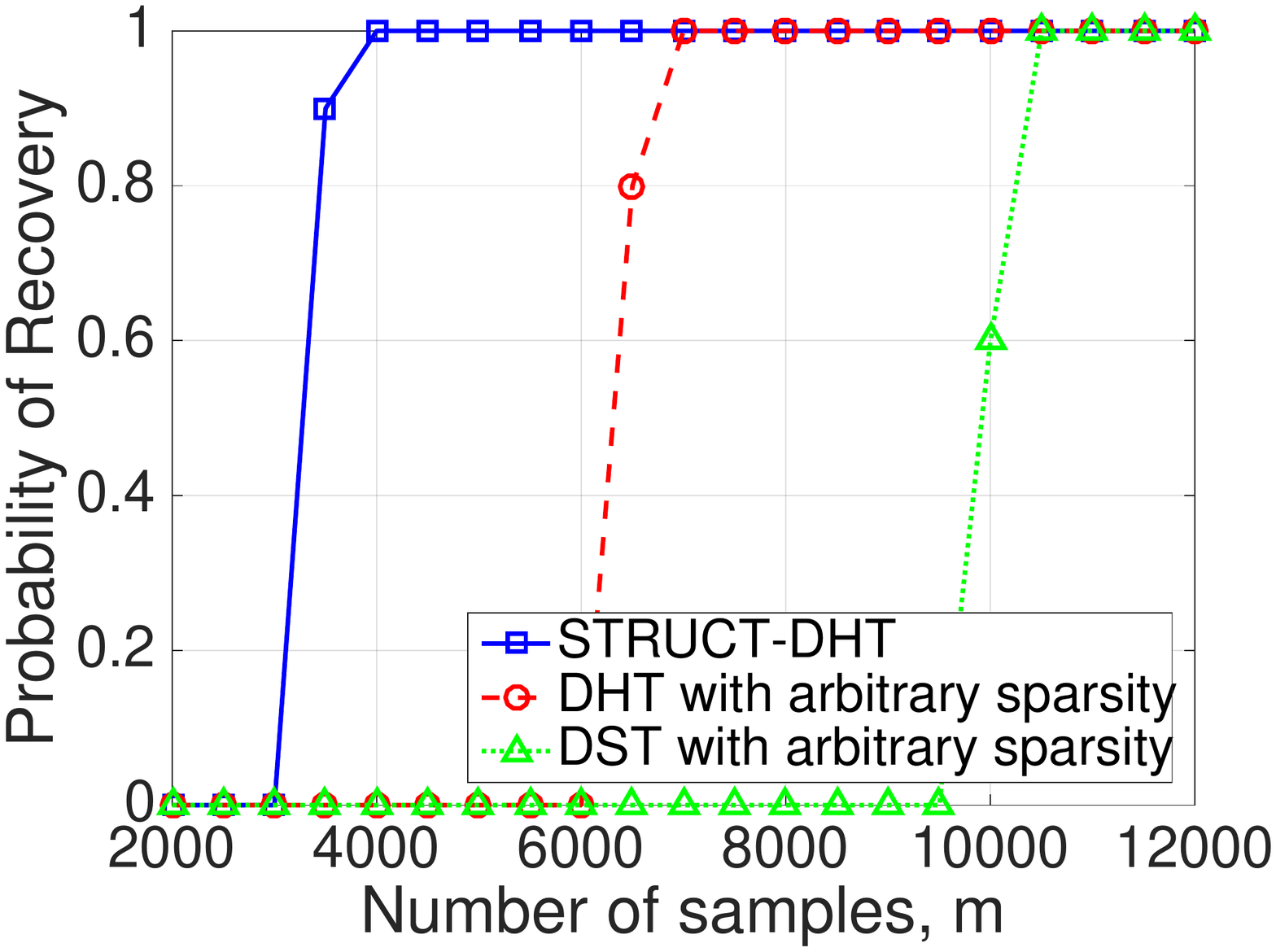}&
\includegraphics[trim = 7mm 58mm 14mm 30mm, clip, width=0.48\linewidth]{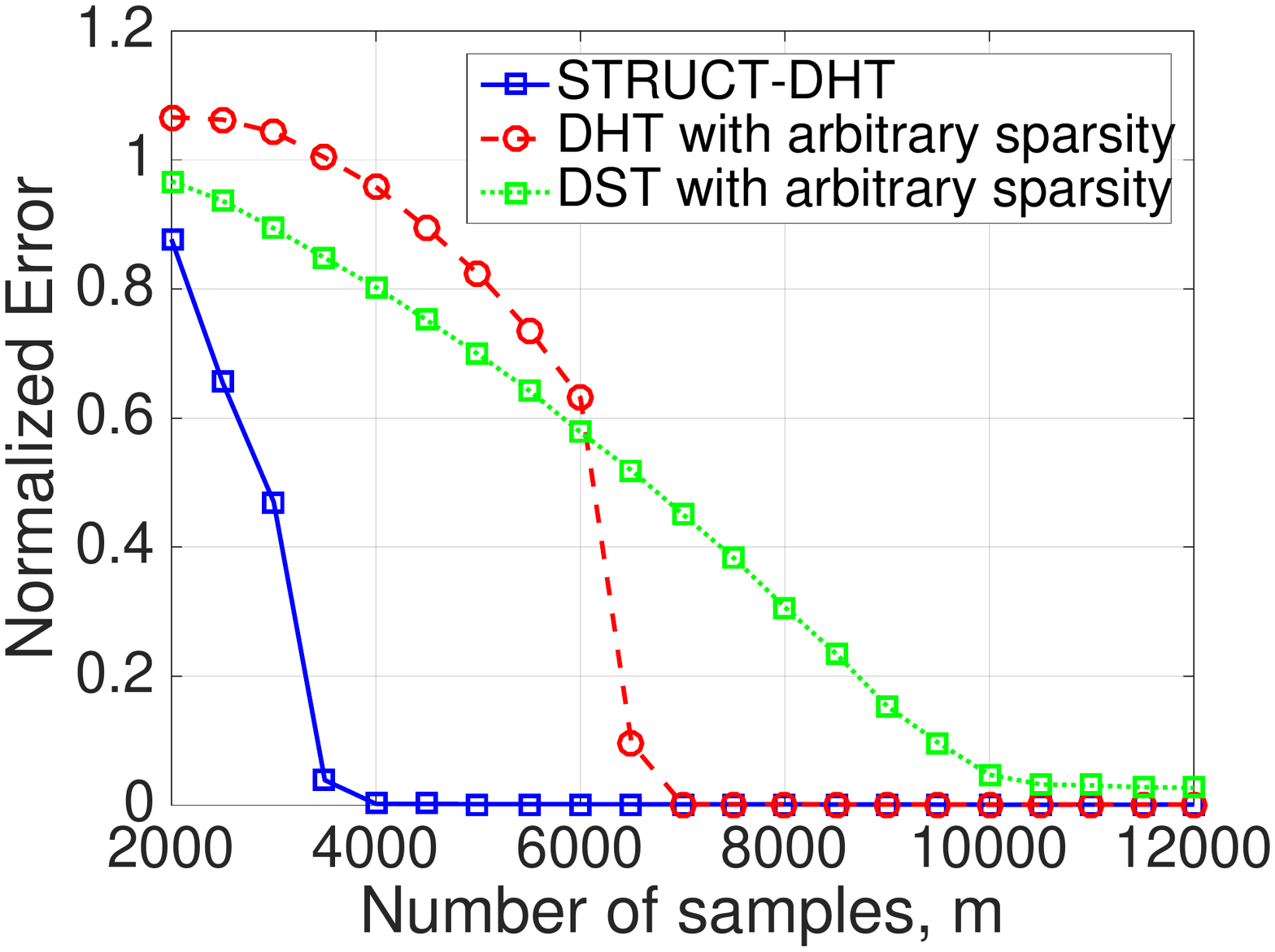}\\
(a) & (b) \\
\includegraphics[trim = 7mm 58mm 15mm 50mm, clip, width=0.48\linewidth]{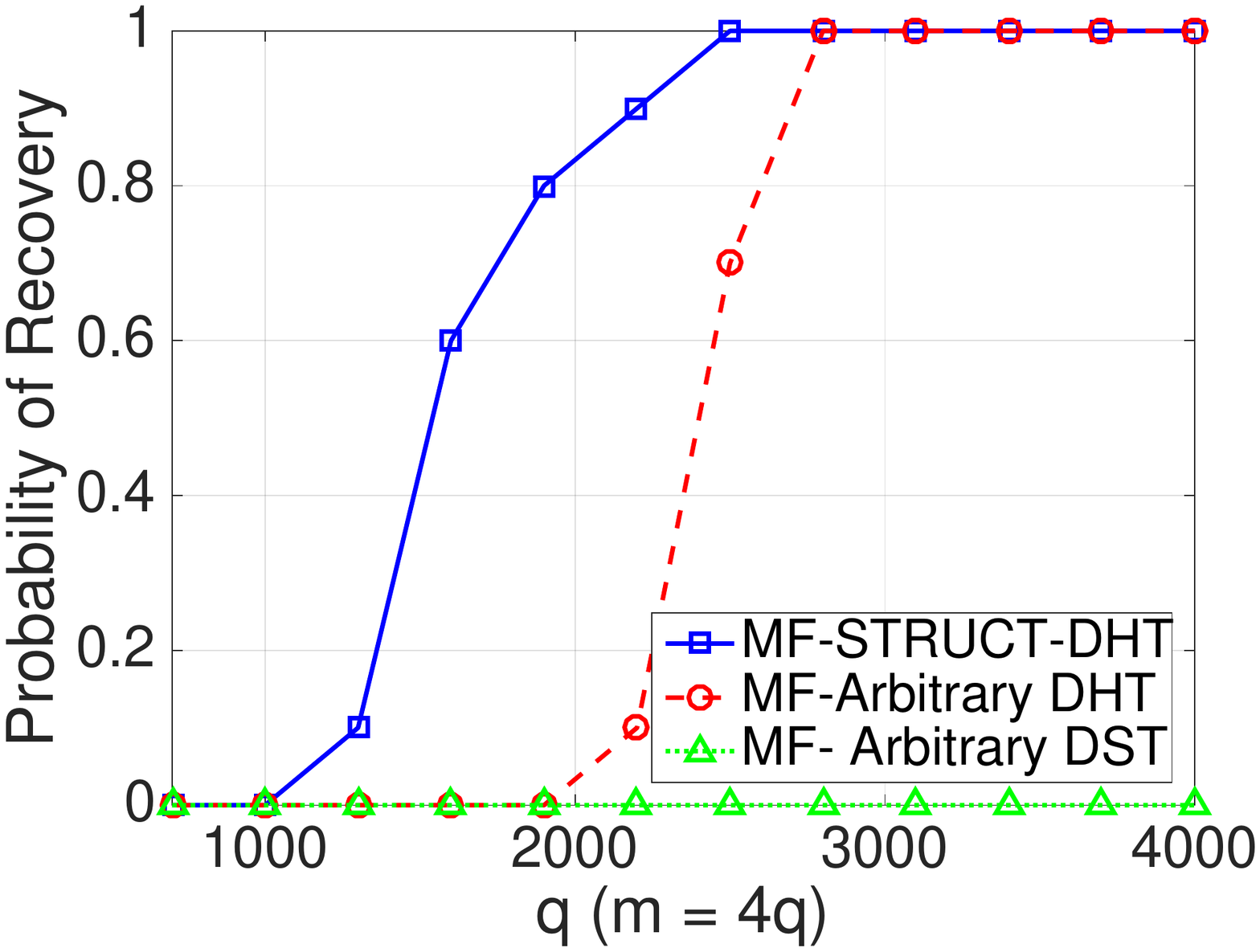} &
\includegraphics[trim = 7mm 58mm 15mm 50mm, clip, width=0.48\linewidth]{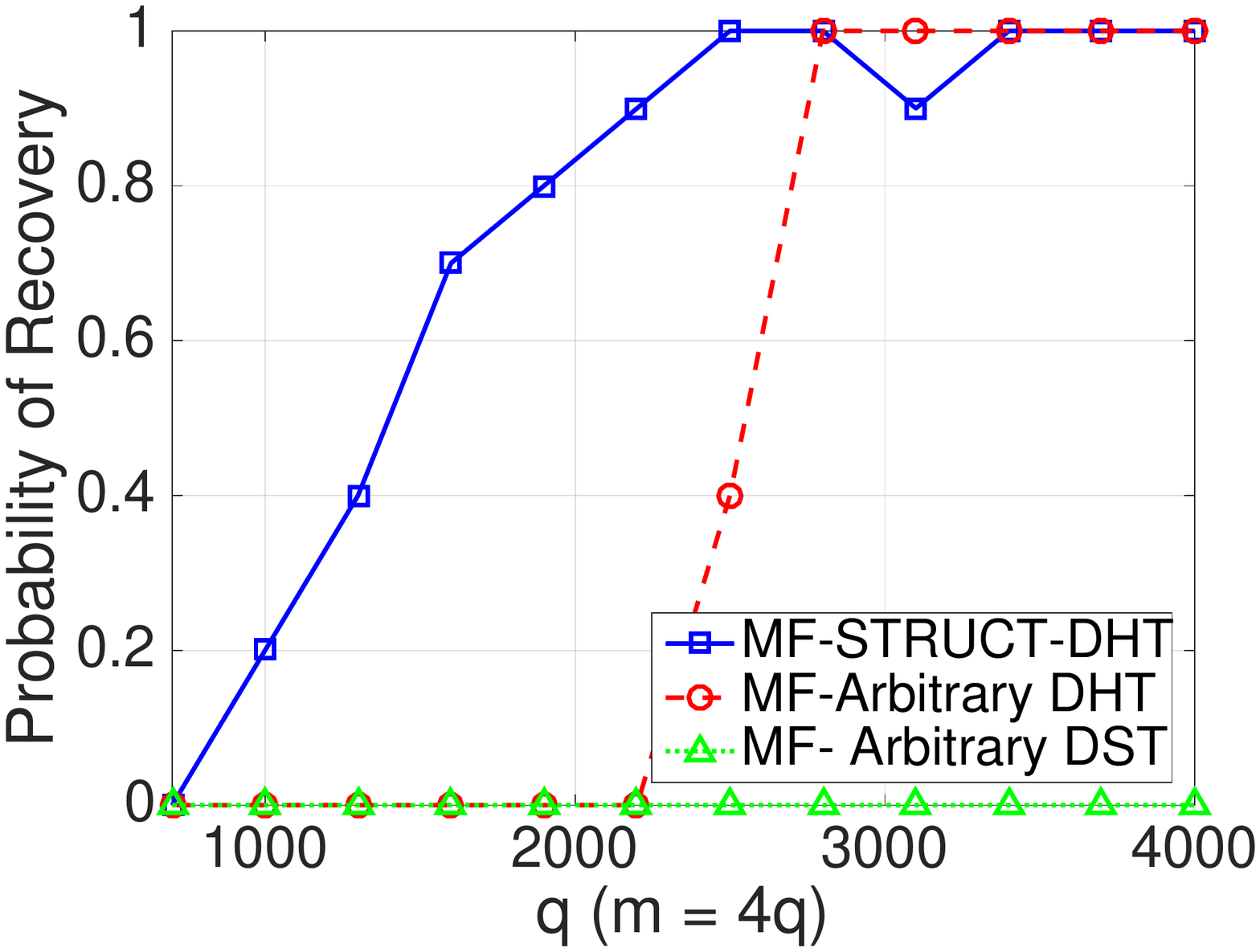}\\
(c) & (d)
\end{tabular}
\endgroup
\end{center}
\vskip  -0.2in
\caption{\small{\emph{Comparison of \textsc{DHT} and \textsc{MF} with structured sparsity with other algorithms. (a) and (b) Probability of recovery in terms of normalized error and Normalized error between $\widehat{\beta} =  \Phi \widehat{\theta_1} + \Psi \widehat{\theta_2}$ and true $\beta$, respectively for $g(x) = \frac{1-e^{-x}}{1 + e^{-x}}$. (c) and (d) Probability of recovery in terms of normalized error for $g(x) = \sin(x)$ and $g(x) = \mod(x)$, respectively. }}}
\label{fig:ComparisonSyn}
\end{figure}

\section{Conclusions}
In this paper, we addressed the problem of demixing from a set of limited nonlinear measurements in high dimensions. Specifically, we considered two nonlinearities: aperiodic and periodic link functions and the structured sparsity in the underlying signal components. For each of these nonlinearities, we proposed an algorithm and support them with sample complexity analysis. As a result of our proposed schemes, we showed that having structured sparsity assumption in the underlying components can significantly reduce the sample complexity compared to the case where we just have regular sparsity prior in these components. Finally, we verified our theoretical claims with some experimental results.  

\small
\bibliographystyle{IEEEbib.bst}
\bibliography{chinbiblio,csbib,mrsbiblio,kernels}

\end{document}